\documentclass[twoside,11pt]{article}
\usepackage{jmlr2e}
\usepackage{graphicx} 
\usepackage{subfigure}
\usepackage[colorlinks=false,allbordercolors={1 1 1}]{hyperref}

\usepackage{natbib}

\usepackage{algorithm, algorithmic}

\usepackage{times, amsfonts, amsmath, amssymb, latexsym, color, graphics, enumerate, amstext,url, epsfig, subfigure}

\jmlrheading{15}{2014}{2213-2238}{12/11; Revised 2/13}{6/14}{Yudong Chen, Ali Jalali, Sujay Sanghavi and Huan Xu}

\ShortHeadings{Clustering Partially Observed Graphs via Convex Optimization}{Chen, Jalali, Sanghavi and Xu}
\firstpageno{2213}

\newcommand{\mat}[1]{{\mathbf #1}}

\newcommand{\M}{\mat{M}}

\newcommand{\A}{\mat{A}}
\newcommand{\I}{\mat{I}}
\newcommand{\B}{\mat{B}}
\newcommand{\U}{\mat{U}}
\newcommand{\K}{\mat{K}}

\newcommand{\X}{\mat{X}}
\newcommand{\Y}{\mat{Y}}

\renewcommand{\P}{\mathbb{P}}
\newcommand{\R}{\mathbb{R}}
\renewcommand{\S}{\mat{S}}
\newcommand{\T}{{\mathcal T}}

\newcommand{\ignore}[1]{}

\newcommand{\PP}{\mathcal{P}}
\newcommand{\II}{\mathcal{I}}
\newcommand{\RR}{{\mathcal R}}

\newcommand{\OmegaO}{ \Omega_{\textrm{obs}} }

\newcommand{\Gammak}{ \Gamma_{k} }
\newcommand{\Gammai}{ \Gamma_{i} }

\newcommand{\W}{ \mat{W}}

\newcommand{\Kmin}{ K_{\min}}
\newcommand{\EE}{\PP_\Omega(\S)}
\newcommand{\llambda}{\lambda}

\begin{document}

\title{Clustering Partially Observed Graphs via Convex Optimization}

\author{\name Yudong Chen \email ydchen@utexas.edu\\
      \name Ali Jalali \email alij@mail.utexas.edu\\
      \name Sujay Sanghavi \email sanghavi@mail.utexas.edu \\
       \addr Department of Electrical and Computer Engineering\\
       The University of Texas at Austin\\
       Austin, TX 78712, USA
       \AND
       \name Huan Xu \email mpexuh@nus.edu.sg \\
       \addr Department of Mechanical Engineering\\
       National University of Singapore\\
       Singapore 117575, SINGAPORE}

\editor{Marina Meila}

\maketitle

\begin{abstract}%
This paper considers the problem of clustering a partially observed unweighted graph---i.e., one where for some node pairs we know there is an edge between them, for some others we know there is no edge, and for the remaining we do not know whether or not there is an edge. We want to organize the nodes into disjoint clusters so that there is relatively dense (observed) connectivity within clusters, and sparse across clusters.

We take a novel yet natural approach to this problem, by focusing on finding the clustering that minimizes the number of ``disagreements''---i.e., the sum of the number of (observed) missing edges within clusters, and (observed) present edges across clusters. Our algorithm uses convex optimization; its basis is a reduction of disagreement minimization to the problem of recovering an (unknown) low-rank matrix and an (unknown) sparse matrix from their partially observed sum. We evaluate the performance of our algorithm on the  classical Planted Partition/Stochastic Block Model. Our main theorem provides sufficient conditions for the success of our algorithm as a function of the minimum cluster size, edge density and observation probability; in particular, the results characterize the tradeoff between the observation probability and the edge density gap. When there are a constant number of clusters of equal size, our results are optimal up to logarithmic factors.

\end{abstract}


\section{Introduction} \label{sec:intro}

This paper is about the following task: given partial observation of an undirected unweighted graph, partition the nodes into disjoint clusters so that there are dense connections within clusters, and sparse connections across clusters. By partial observation, we mean that for some node pairs we know if there is an edge or not, and for the other node pairs we do not know---these pairs are {\em unobserved}. This problem arises in several fields across science and engineering. For example, in sponsored search, each cluster is a submarket that represents a specific group of advertisers that do most of their spending on a group of query phrases---see e.g.,~\cite{yaho} for such a project at Yahoo. In VLSI and design automation, it is useful in minimizing signaling between components~\citep{kernighan_lin}. In social networks, clusters may represent groups of people with similar interest or background; finding clusters enables better recommendations and link prediction~\citep{Social1}. In the analysis of document databases, clustering the citation graph is often an essential and informative first step~\citep{db1}. In this paper, we will focus not on specific application domains, but rather on the basic graph clustering problem itself.

Partially observed graphs appear in many applications. For example, in online social networks like Facebook, we observe an edge/no edge between two users when they accept each other as a friend or explicitly decline a friendship suggestion. For the other user pairs, however, we simply have no friendship information between them, which are thus unobserved. More generally, we have partial observations whenever obtaining similarity data is difficult or expensive (e.g., because it requires human participation). In these applications, it is often the case that \emph{most} pairs are unobserved, which is the regime we are particularly interested in.

As with any clustering problem, we need a precise mathematical definition of the clustering criterion with potentially a guaranteed performance. There is relatively few existing results with provable performance guarantees for graph clustering with partially observed node pairs. Many existing approaches to clustering fully observed graphs either require an additional input (e.g., the number of clusters $k$ required for spectral or $k$-means clustering methods), or do not guarantee the performance of the clustering. We review existing related work in Section~\ref{sec:related}.

\subsection{Our Approach}\label{sec:intro_approach}
We focus on a natural formulation, one that {does not require any other extraneous input} besides the graph itself. It is based on minimizing {\em disagreements}, which we now define. Consider any candidate clustering; this will have {(a)} observed node pairs that are in different clusters, but have an edge between them, and {(b)} observed node pairs that are in the same cluster, but do not have an edge between them. The total number of node pairs of types {(a)} and {(b)} is the number of disagreements between the clustering and the given graph. We focus on the problem of finding the {optimal clustering}---one that minimizes the number of disagreements. Note that we do {\em not} pre-specify the number of clusters. For the special case of fully observed graphs, this formulation is exactly the same as the problem of \emph{correlation clustering}, first proposed by~\cite{banblucha}. They show that exact minimization of the above objective is NP-hard in the worst case---we survey and compare with this and other related work in Section~\ref{sec:related}. As we will see, our approach and results are different.

We aim to achieve the combinatorial disagreement minimization objective using matrix splitting via convex optimization. In particular, as we show in Section~\ref{sec:main_results} below, one can represent the adjacency matrix of the given graph as the sum of an unknown low-rank matrix (corresponding to ``ideal" clusters) and a sparse matrix (corresponding to disagreements from this ``ideal'' in the given graph). Our algorithm either returns a clustering, which is guaranteed to be disagreement minimizing, or returns a ``failure"---it never returns a sub-optimal clustering. For our main analytical result, we evaluate our algorithm's performance on the natural and classical \emph{planted partition/stochastic block model} with partial observations.  Our analysis provides stronger guarantees than are current results on general matrix splitting \citep{candeswrightma,tong,xiaodongli,CHEJALSANCAR11}. The algorithm, model and results are given in Section~\ref{sec:main_results}. We prove our theoretical results in Section~\ref{sec:proofs} and provide empirical results in Section~\ref{sec:exper}.

\subsection{Related Work}
\label{sec:related}

Our problem can be interpreted in the general clustering context as one in which the presence of an edge between two points indicates a ``similarity'', and the lack of an edge means ``no similarity''. 
The general field of clustering is of course vast, and a detailed survey of all methods therein is beyond our scope here. We focus instead on the three sets of papers most relevant to the problem here: the work on correlation clustering, on the planted partition/stochastic block model, and on graph clustering with partial observations.

\subsubsection{Correlation Clustering} As mentioned, for a completely observed graph, our problem is mathematically precisely the same as correlation clustering formulated in~\cite{banblucha}; in particular a ``+" in correlation clustering corresponds to an edge in the graph,  a ``-" to the lack of an edge, and disagreements are defined in the same way. Thus, this paper can equivalently be considered as an algorithm, and guarantees, for {\em correlation clustering under partial observations}. Since correlation clustering is NP-hard, there has been much work on devising alternative approximation algorithms~\citep{banblucha, emafia}. Approximations using convex optimization, including LP relaxation~\citep{chagurwir,demimm,demimmemafia} and SDP relaxation~\citep{swa,MATSCH10}, possibly followed by rounding, have also been developed. We emphasize that we use a different convex relaxation, and we focus on understanding when our convex program yields an optimal clustering without further rounding. 

We note that~\citet{MATSCH10} use a convex formulation with constraints enforcing positive semi-definiteness, triangle inequalities and fixed diagonal entries. For the fully observed case, their relaxation is at least as tight as ours, and since they add more constraints, it is possible that there are instances where their convex program works and ours does not. However, this seems hard to prove/disprove. Indeed, in the full observation setting they consider, their exact recovery guarantee is no better than ours. Moreover, as we argue in the next section, our guarantees are order-wise optimal in some important regimes and thus cannot be improved even with a tighter relaxation. Practically, our method is faster since, to the best of our knowledge, there is no low-complexity algorithm to deal with the $ \Theta(n^3) $ triangle inequality constraints required by~\citet{MATSCH10}. This means that our method can handle large graphs while their result is practically restricted to small ones ($\sim$ 100 nodes). In summary, their approach has higher computational complexity, and does not provide significant and characterizable performance gain in terms of exact cluster recovery.

\subsubsection{Planted Partition Model} The planted partition model, also known as the stochastic block-model~\citep{CondonKarp,holland}, assumes that the graph is generated with in-cluster edge probability $ p $ and inter-cluster edge probability $ q $ (where $ p>q $) and fully observed. The goal is to recover the latent cluster structure. A class of this model with $ \tau \triangleq \max\{1-p,q\} < \frac{1}{2} $ is often used as benchmark for \emph{average case} performance for correlation clustering~\citep[see, e.g.,][]{MATSCH10}. Our theoretical results are applicable to this model and thus directly comparable with existing work in this area. A detailed comparison is provided in Table \ref{tab:comparison}. For fully observed graphs, our result matches the previous best bounds in both the minimum cluster size and the difference between in-cluster/inter-cluster densities. We would like to point out that nuclear norm minimization has been used to solve the closely related planted clique problem~\citep{Alon98findinga,AMEVAV11}.

\begin{table}[h!]
\begin{center}
\begin{tabular}{|c|c|c|}
\hline 
Paper & Cluster size $K$ & Density difference $(1-2\tau)$  \tabularnewline
\hline 
\hline 
\cite{boppana1987eigenvalues} & $n/2$ & $\tilde{\Omega}(\frac{1}{\sqrt{n}})$  \tabularnewline
\hline 
\cite{jerrum1998metropolis} & $n/2$ & $\tilde{\Omega}(\frac{1}{n^{1/6-\epsilon}})$  \tabularnewline
\hline 
\cite{CondonKarp} & $\tilde{\Omega}(n)$ & $\tilde{\Omega}(\frac{1}{n^{1/2-\epsilon}})$  \tabularnewline
\hline 
\cite{carson2001planted} & $n/2$ & $\tilde{\Omega}(\frac{1}{\sqrt{n}})$  \tabularnewline
\hline 
\cite{feige2001semirandom} & $n/2$ & $\tilde{\Omega}(\frac{1}{
\sqrt{n}}$)  \tabularnewline
\hline 
\cite{mcsherry2001spectralpartitioning} & $\tilde{\Omega}(n^{2/3})$ & $\tilde{\Omega}(\sqrt{\frac{n^2}{K^3}})$ \tabularnewline
\hline 
\cite{bollobas2004maxcut} & $\tilde{\Omega}(n)$  & $\tilde{\Omega}(\sqrt{\frac{ 1}{n}})$   \tabularnewline
\hline
\cite{giesen2005manypartition} & $\tilde{\Omega}(\sqrt{n})$ & $\tilde{\Omega}(\frac{\sqrt{n}}{K})$  \tabularnewline
\hline 
\cite{shamir2007improved} & $\tilde{\Omega}(\sqrt{n} )$ & $\tilde{\Omega}(\frac{\sqrt{n}}{K})$  \tabularnewline
\hline 
\cite{MATSCH10} & $\tilde{\Omega}(\sqrt{n})$  & $\tilde{\Omega}(1)$   \tabularnewline
\hline
\cite{Rohe10} & $\tilde{\Omega}(n^{3/4})$ & $\tilde{\Omega}(\frac{n^{3/4}}{K})$  \tabularnewline
\hline
\cite{OYMHAS11} & $\tilde{\Omega}(\sqrt{n})$ & $\tilde{\Omega}(\frac{\sqrt{n}}{K})$  \tabularnewline
\hline 
\cite{chaudhuri} & $\tilde{\Omega}(\sqrt{n})$ &  $\tilde{\Omega}(\frac{\sqrt{n}}{K})$ \tabularnewline
\hline 
 &  &   \tabularnewline
\hline
This paper & $\tilde{\Omega}(\sqrt{n})$ & $\tilde{\Omega}(\frac{\sqrt{n}}{K})$  \tabularnewline
\hline 
\end{tabular}
\end{center}
\caption{\small{\it Comparison with literature.} This table shows the lower-bound requirements on the minimum cluster size $K$ and the density difference $p-q=1-2\tau$ that existing literature needs for exact recovery of the planted partitions, when the graph is fully observed and $ \tau \triangleq \max\{1-p,q\}=\Theta(1) $. Some of the results in the table only guarantee recovering the membership of most, instead of all, nodes. To compare with these results, we use the soft-$ \Omega $ notation $ \tilde{\Omega}(\cdot) $, which hides the logarithmic factors that are necessary for recovering all nodes, which is the goal of  this paper. 
\label{tab:comparison}}

\end{table}

\subsubsection{Partially Observed Graphs} The previous work listed in Table \ref{tab:comparison}, except~\cite{OYMHAS11}, does not handle partial observations directly. One natural way to proceed is to impute the missing observations with no-edge, or random edges with symmetric probabilities, and then apply any of the results in Table \ref{tab:comparison}. This approach, however, leads to sub-optimal results. Indeed, this is done explicitly by \citet{OYMHAS11}. They require the probability of observation $ p_0 $ to satisfy $ p_0 \gtrsim \frac{\sqrt{\Kmin}}{n} $, where $ n $ is the number of nodes and $ \Kmin $ is the minimum cluster size; in contrast, our approach only needs $ p_0 \gtrsim \frac{n}{\Kmin^2}$ (both right hand sides have to be less than $ 1 $, requiring $ \Kmin \gtrsim \sqrt{n} $,  so the right hand side of our condition is order-wise smaller and thus less restrictive.)  \cite{budget} deal with partial observations directly and shows that $ p_0 \gtrsim \frac{1}{n} $ suffices for recovering two clusters of size $ \Omega(n) $. Our result is applicable to much smaller clusters of size $\tilde{\Omega}(\sqrt{n}) $. In addition, a nice feature of our result is that it explicitly characterizes the tradeoffs between the three relevant parameters: $ p_0 $, $ \tau $, and $ \Kmin $; theoretical result like this is not available in previous work. 

There exists other work that considers partial observations, but under rather different settings. For example, \cite{balcangupta}, \cite{voevodski2010efficient} and \cite{krishnamurthy2012hierarchical} consider the clustering problem where one samples the rows/columns of the adjacency matrix  rather than its entries. \cite{hunter} consider partial observations in the features rather than in the similarity graph. \cite{eriksson} show that $ \tilde{\Omega}(n) $ actively selected pairwise similarities are sufficient for recovering a hierarchical clustering structure. Their results seem to rely on the hierarchical structure. When disagreements are present, the first split of the cluster tree can recovers clusters of size $ \Omega(n) $; our results allow smaller clusters. Moreover, they require active control over the observation process, while we assume random observations.

\section{Main Results}
\label{sec:main_results}

Our setup for the graph clustering problem is as follows. We are given a partially observed graph of $ n $ nodes, whose adjacency matrix is $\A \in \mathbb{R}^{n\times n}$, which has $a_{i,j} = 1$ if there is an edge between nodes $i$ and $j$, $a_{i,j} = 0$ if there is no edge, and $a_{i,j} = $``?'' if we do not know. (Here we follow the convention that $a_{i,i}=0$ for all $i$.) Let $\OmegaO\triangleq\{(i,j):a_{i,j}\neq ?\}$ be the set of observed node pairs. The goal to find the {optimal clustering}, i.e., the one that has the minimum number of disagreements (defined in Section~\ref{sec:intro_approach}) in $\OmegaO$.

In the rest of this section, we present our algorithm for the above task and analyze its performance under the planted partition model with partial observations. We also study the optimality of the performance of our algorithm by deriving a necessary condition for any algorithm to succeed.

\subsection{Algorithm}

Our algorithm is based on convex optimization, and either {(a)} outputs a clustering that is guaranteed to be the one that minimizes the number of observed disagreements, or {(b)} declares ``failure".
In particular, it never produces a suboptimal clustering.\footnote{In practice, one might be able to use the ``failed'' output with rounding as an approximate solution. In this paper, we focus on the performance of the unrounded algorithm.} We now briefly present the main idea and then describe the algorithm.

Consider first the fully observed case, i.e., every $a_{i,j} = 0$ or 1. Suppose also that the graph is already ideally clustered---i.e., there is a partition of the nodes such that there is no edge between clusters, and each cluster is a clique. In this case,  the matrix $\A+\I$ is now a {\em low-rank} matrix, with the rank equal to the number of clusters. This can be seen by noticing that if we re-order the rows and columns so that clusters appear together, the result would be a {\em block-diagonal} matrix, with each block being an all-ones submatrix---and thus rank one. Of course, this re-ordering does not change the rank of the matrix, and hence $\A+\I$ is exactly low-rank.

Consider now any given graph, still fully observed. In light of the above, we are looking for a decomposition of its $\A+\I$ into a low-rank part $\K^*$ (of block-diagonal all-ones, one block for each cluster) and a remaining $\B^*$ (the disagreements), such that the number of non-zero entries in $\B^*$ is as small as possible; i.e., $\B^*$ is sparse. Finally, the problem we look at is recovery of the best $\K^*$ when we do not observe all entries of $ \A+\I $. The idea is depicted in Figure~\ref{fig1}.

\begin{figure}[t]
\centering
\includegraphics[width=2.2in,height=1.9in,trim = 0mm 40mm 0mm 40mm, clip]{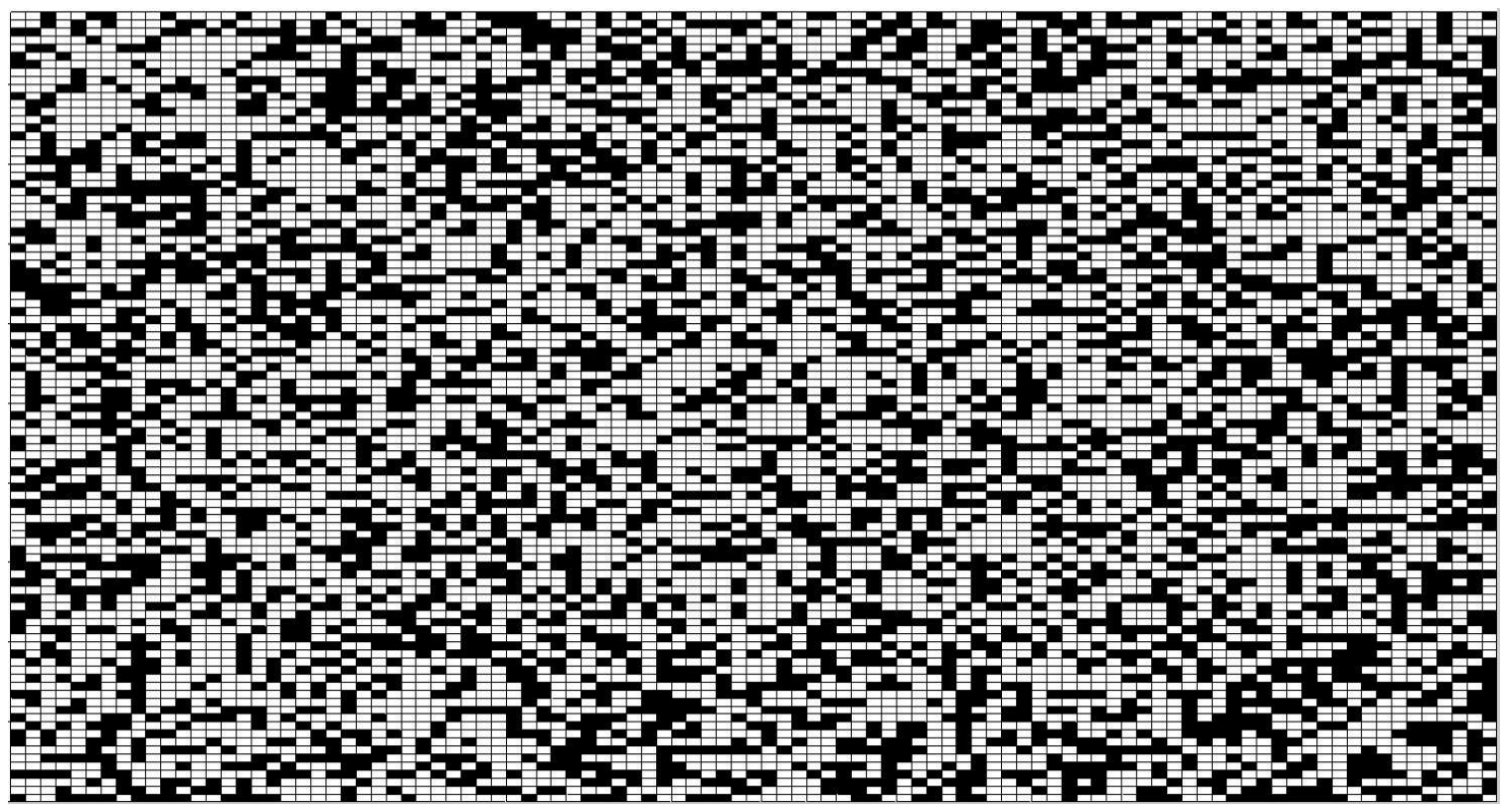}
\includegraphics[width=2.2in,height=1.9in,trim = 0mm 40mm 0mm 40mm, clip]{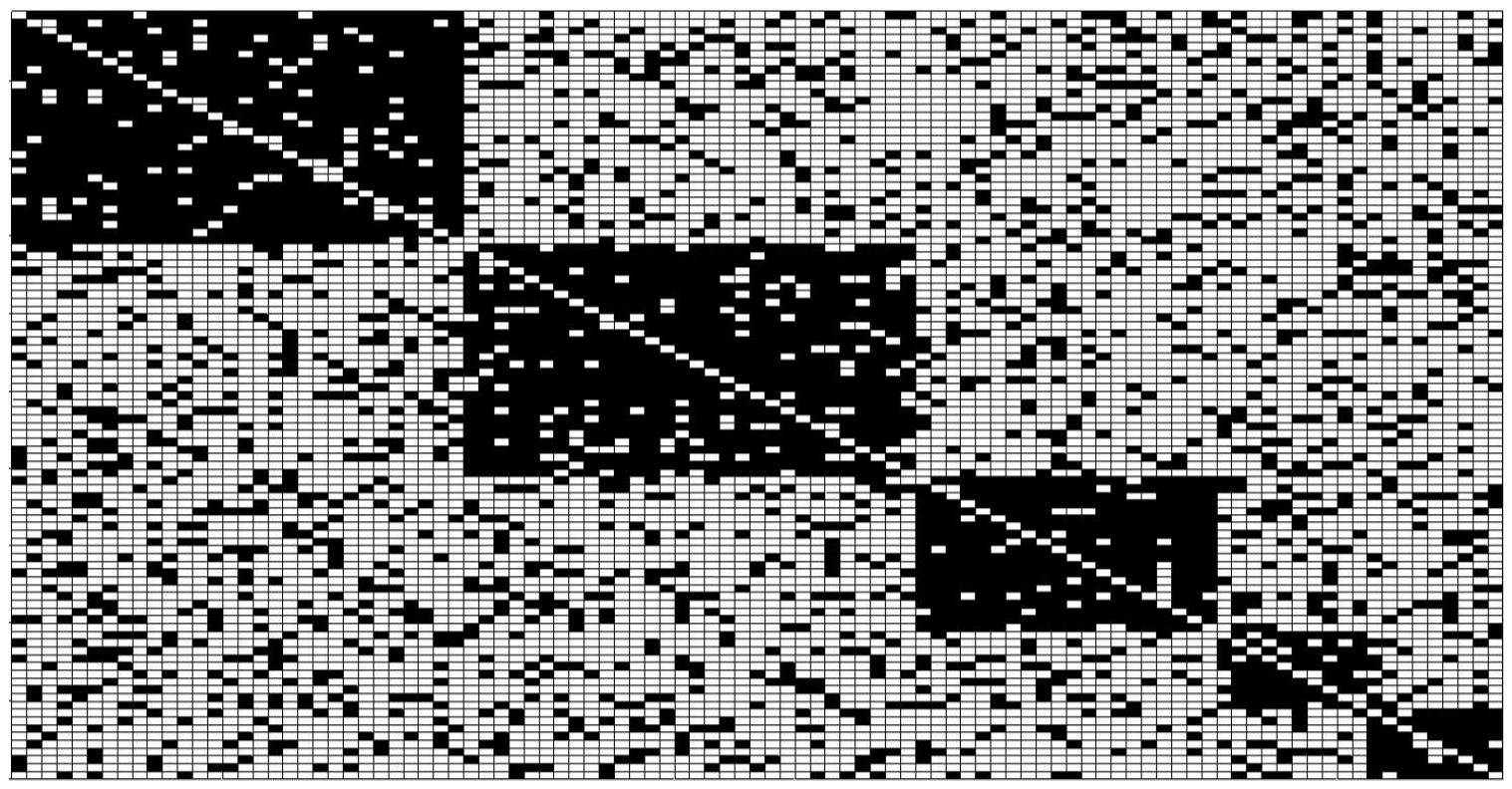}
\caption{\small The adjacency matrix of a graph drawn from the planted partition model before and after proper reordering (i.e., clustering) of the nodes. The figure on the right is indicative of the matrix as a superposition of a sparse matrix and a low-rank block diagonal one.\label{fig1}}
\end{figure}

We propose to perform the matrix splitting using convex optimization~\citep{cspw,candeswrightma}. Our approach consists of {dropping} any additional structural requirements, and just looking for a decomposition of the given $\A+\I$ as the sum of a sparse matrix $\B$ and a low-rank matrix $\K$. Recall that $\OmegaO$ is the set of observed entries, i.e., the set of elements of $\A$ that are known to be 0 or 1; we use the following convex program:
\begin{equation}\label{eq:partial_conv_prog} 
\begin{aligned}
\min_{\B,\K} & \quad  \lambda \left\Vert\B\right\Vert_1 + \left\Vert\K\right\Vert_* \\
\text{s.t.}&   \quad \mathcal{P}_{\Omega_{obs}}(\B+\K) = \mathcal{P}_{\Omega_{obs}}(\A+\I).
\end{aligned}
\end{equation}
Here, for any matrix $\M$, the term $\mathcal{P}_{\OmegaO}(\M)$ keeps all elements of $\M$ in $\OmegaO$ unchanged, and sets all other elements to 0; the constraints thus state that the sparse and low-rank matrix should in sum be consistent with the observed entries of $ \A+\I $. The term $\Vert\B\Vert_1 = \sum_{i,j} |b_{i,j}|$ is the $\ell_1$ norm of the entries of the matrix $ \B $, which is well-known to be a convex surrogate for the number of non-zero entries $\Vert\B\Vert_0$. The second term  $\Vert\K\Vert_* = \sum_s \sigma_s(\K)$ is the nuclear norm (also known as the trace norm), defined as the sum of the singular values of $\K$. This has been shown to be the tightest convex surrogate for the rank function for matrices with unit spectral norm \citep{fazelphd}. Thus our objective function is a convex surrogate for the (natural) combinatorial objective $\lambda\Vert\B\Vert_0 +  \textrm{rank}(\K)$. The optimization problem~\eqref{eq:partial_conv_prog} is, in fact, a semidefinite program (SDP)~\citep{cspw}.

We remark on the above formulation. (a) This formulation does not require specifying the number of clusters; this parameter is effectively learned from the data. The tradeoff parameter $ \lambda $ is artificial and can be easily determined: since any desired $ \K^* $ has trace exactly equal to $ n $, we simply choose the smallest $ \lambda $ such that the trace of the optimal solution is at least $ n $. This can be done by, e.g., bisection, which is described below. (b) It is possible to obtain tighter convex relaxations by adding more constraints, such as the diagonal entry constraints $ k_{i,i}=1,\forall i$,  the positive semi-definite constraint $\K\succeq 0$, or even the triangular inequalities $ k_{i,j} + k_{j,k} -k_{i,k} \le 1 $. Indeed, this is done by \cite{MATSCH10}. Note that the guarantees for our formulation (to be presented in the next subsection) automatically imply guarantees for any other tighter relaxations. We choose to focus on our (looser) formulation for two reasons. First, and most importantly, even with the extra constraints, \cite{MATSCH10} do not deliver better exact recovery guarantees (cf.~Table~\ref{tab:comparison}). In fact, we show in Section~\ref{sec:lowerbounds} that our results are near optimal in some important regimes, so tighter relaxations do not seem to provide additional benefits in exact recovery. Second, our formulation can be solved efficiently using existing Augmented Lagrangian Multiplier methods \citep{LinALM}. This is no longer the case with the $ \Theta(n^3) $ triangle inequality constraints enforced by~\citet{MATSCH10}, and solving it as a standard SDP is only feasible for small graphs.   

We are interested in the case when the convex program~\eqref{eq:partial_conv_prog} produces an optimal solution $ \K$ that is a block-diagonal matrix and corresponds to an ideal clustering.
\begin{definition}[Validity] The convex program~\eqref{eq:partial_conv_prog} is said to produce a {\em valid} output if the low-rank matrix part $\K$ of the optimal solution corresponds to a graph of disjoint cliques; i.e., its rows and columns can be re-ordered to yield a block-diagonal matrix with all ones for each block.
\end{definition}
Validity of a given $\K$ can be easily checked via elementary re-ordering operations.\footnote{If we re-order a valid $ \K $ such that identical rows and columns appear together, it will become block-diagonal.} Our first simple but useful insight is that whenever the convex program~(\ref{eq:partial_conv_prog}) yields a valid solution, it is the disagreement minimizer.
\begin{theorem}
For any $\lambda > 0$, if the solution of~\eqref{eq:partial_conv_prog} is valid, then it is the clustering that minimizes the number of observed disagreements.
\label{thm:disamin}
\end{theorem}
Our complete clustering procedure is given as Algorithm~\ref{alg1}. It takes the adjacency matrix of the graph $\A$ and outputs either the optimal clustering or declares failure.
Setting the parameter $\lambda$ is done via binary search. The initial value of $ \lambda $ is not crucial; we use $ \lambda = \frac{1}{32\sqrt{\bar{p}_0 n}} $ based on our theoretical analysis in the next sub-section, where $ \bar{p}_0 $ is the empirical fraction of observed pairs. To solve the optimization problem~\eqref{eq:partial_conv_prog}, we use the fast algorithm developed by \citet{LinALM}, which is tailored for matrix splitting and takes advantage of the sparsity of the observations. By Theorem~\ref{thm:disamin}, whenever the algorithm results in a valid $\K$, we have found the optimal clustering.
\begin{algorithm}
\caption{Optimal-Cluster($\A$)}
\label{alg1}
\begin{algorithmic}
\STATE $ \lambda \gets \frac{1}{32\sqrt{\bar{p}_0 n}} $
\WHILE {not terminated}
\STATE Solve~\eqref{eq:partial_conv_prog} to obtain the solution $ (\B,\K) $
\IF {$\K$ is valid}
	\STATE Output the clustering in $\K$ and EXIT.
\ELSIF {$ \textrm{trace}(\K) > n $}
	\STATE $ \lambda \gets \lambda/2  $
\ELSIF {$ \textrm{trace}(\K) < n $}
	\STATE $ \lambda \gets 2\lambda  $
\ENDIF

\ENDWHILE
\STATE Declare Failure.
\end{algorithmic}
\end{algorithm}

\subsection{Performance Analysis}
\label{sec:performance}

For the main analytical contribution of this paper, we provide conditions under which the above algorithm will find the clustering that minimizes the number of disagreements among the observed entries. In particular, we characterize its performance under the standard and classical planted partition/stochastic block model with partial observations, which we now describe.
\begin{definition}
[Planted Partition Model with Partial Observations] Suppose that $n$ nodes are partitioned into $ r $ clusters, each of size at least $ \Kmin $. Let $\K^*$ be the low-rank matrix corresponding to this clustering (as described above). The adjacency matrix $ \A $ of the graph is generated as follows: for each pair of nodes $ (i,j) $ in the same cluster, $ a_{i,j} =?$  with probability $ 1-p_0 $, $ a_{i,j}=1 $ with probability $ p_0 p $, or $ a_{ij}=0 $ otherwise, independent of all others; similarly, for $ (i,j) $ in different clusters, $ a_{i,j} =?$  with probability $ 1-p_0 $, $ a_{i,j}=1 $ with probability $ p_0 q $, or $ a_{i,j}=0 $ otherwise. 
\end{definition}
Under the above model, the graph is observed at locations chosen at random with probability $ p_0 $. In expectation a fraction of $ 1-p $ of the in-cluster observations are disagreements; similarly, the fraction of disagreements in the across-cluster observations is $ q $. Let $\B^* = \mathcal{P}_{\OmegaO} (\A + \I - \K^*)$ be the matrix of observed disagreements for the original clustering; note that the support of $\B^*$ is contained in $\OmegaO$.  The following theorem provides a sufficient condition for our algorithm to recover the original clustering $ (\B^*,\K^*) $ with high probability. Combined with Theorem~\ref{thm:disamin}, it also shows that under the same condition,  the original clustering is disagreement minimizing with high probability.
\begin{theorem}\label{thm:partial}
Let $ \tau = \max\{1-p,q\} $. There exist universal positive constants $ c $ and $C$ such
that, with probability at least $1-cn^{-10}$, the original clustering
$(\B^\ast,\K^\ast)$ is the unique optimal solution of
(\ref{eq:partial_conv_prog}) with $\lambda = \frac{1}{32\sqrt{np_0}}$
provided that
\begin{equation}
 p_{0}\left(1-2\tau\right)^{2}\ge C\frac{n\log^{2}n}{\Kmin^{2}}.\label{eq:main_cond}
\end{equation}
\end{theorem}
Note that the quantity $ \tau $ is (an upper bound of) the probability of having a disagreement, and  $ 1-2\tau $ is (a lower bound of) the density gap $ p-q $. The sufficient condition in the theorem is given in terms of the three parameters that define problem: the minimum cluster size $ \Kmin $, the density gap $ 1-2\tau$, and the observation probability $ p_0 $. We remark on these parameters.

\begin{itemize}
\item \textit{Minimum cluster size $ \Kmin $.} Since the left hand side of the condition~\eqref{eq:main_cond} in Theorem~\ref{thm:partial} is no more than $ 1 $, it imposes a lower-bound $ \Kmin=\tilde{\Omega}(\sqrt{n}) $ on the cluster sizes. This means that our method can handle a growing number $ \tilde{O}(\sqrt{n})$ of clusters. The lower-bound on $ \Kmin $ is attained when $ 1-2\tau$ and $p_0 $ are both $ \Theta(1) $, i.e., not decreasing as $ n $ grows. Note that all relevant works require a lower-bound at least as strong as ours (cf. Table~\ref{tab:comparison}).

\item \textit{Density gap $ 1-2\tau $.}  When $ p_0 = \Theta(1) $, our result allows this gap to be vanishingly small, i.e., $ \tilde{\Omega}\left(\frac{\sqrt{n }}{\Kmin}\right) $, where a larger $ \Kmin $ allows for a smaller gap. As we mentioned in Section~\ref{sec:related}, this matches the best available results (cf. Table~\ref{tab:comparison}), including those in \citet{MATSCH10} and \citet{OYMHAS11}, which use tighter convex relaxations that are more computationally demanding. We note that directly applying existing results in the low-rank-plus-sparse literature \citep{candeswrightma,xiaodongli} leads to weaker results, where the gap be bounded below by a constant.

\item \textit{Observation probability $ p_0 $.} When $ 1-2\tau =\Theta(1) $, our result only requires a vanishing fraction of observations, i.e., $ p_0 $ can be as small as $ \tilde{\Theta}\left(\frac{n}{\Kmin^2}\right)$; a larger $ \Kmin $ allows for a smaller $ p_0 $. As mentioned in Section~\ref{sec:related}, this scaling is better than prior results we know of. 

\item \textit{Tradeoffs.} A novel aspect of our result is that it shows an explicit tradeoff between the observation probability $ p_0 $ and the density gap $ 1-2\tau $. The left hand side of~\eqref{eq:main_cond} is linear in $ p_0 $ and quadratic in $ 1-2\tau  $. This means if the number of observations is four times larger, then we can handle a $ 50\% $ smaller density gap. Moreover, $ p_0 $ can go to zero quadratically faster then $ 1-2\tau $. Consequently, treating missing observations as disagreements would lead to quadratically weaker results. This agrees with the intuition that handling missing entries with known locations is easier than correcting disagreements whose locations are unknown.
\end{itemize}

We would like to point out that our algorithm has the capability to handle outliers. Suppose there are some isolated nodes which do not belong to any cluster, and they connect to each other and each node in the clusters with probability at most $ \tau $, with $ \tau $ obeying the condition~\eqref{eq:main_cond} in Theorem~\ref{thm:partial}. Our algorithm will classify all these edges as disagreements, and hence automatically reveal the identity of each outlier. In the output of our algorithm, the low rank part $\K$ will have all zeros in the  columns and rows corresponding to outliers---all their edges will appear in the disagreement matrix $\B$.

\subsection{Lower Bounds}
\label{sec:lowerbounds}

We now discuss the tightness of Theorem~\ref{thm:partial}. Consider first the case where $ \Kmin=\Theta(n) $, which means there are a constant number of clusters. We establish a fundamental lower bound on the density gap $ 1-2\tau $ and the observation probability $ p_0 $ that are  required for \emph{any} algorithm to correctly recover the clusters. 

\begin{theorem}
\label{thm:lowerbounds}
Under the planted partition model with partial observations, suppose the true clustering is chosen uniformly at random from all possible clusterings with equal cluster size $ K $. If $ K=\Theta(n) $ and $ \tau =1-p=q>1/100 $, then for any algorithm to correctly identify the clusters with probability at least $ \frac{3}{4} $, we need 
$$
p_0 (1-2\tau)^2 \ge C \frac{1}{n},
$$
where $ C>0 $ is an absolute constant.
\end{theorem}

Theorem~\ref{thm:lowerbounds} generalizes a similar result in~\citet{chaudhuri}, which does not consider partial observations. The theorem applies to any algorithm regardless of its computational complexity, and characterizes the fundamental tradeoff between $ p_0 $ and $ 1-2\tau $. It shows that when $ \Kmin = \Theta(n) $, the requirement for $ 1-2\tau $ and $ p_0 $ in Theorem~\ref{thm:partial} is optimal up to logarithmic factors, and cannot be significantly improved by using more complicated methods.

For the general case with $ \Kmin = O(n) $, only part of the picture is known. Using non-rigorous arguments, \cite{decelle} show that  $ 1-2\tau \gtrsim \frac{\sqrt{n}}{\Kmin} $ is necessary when $ \tau=\Theta(1) $ and the graph is fully observed; otherwise recovery is impossible or computationally hard. According to this lower-bound, our requirement on the density gap $ 1-2\tau $ is probably tight (up to log factors) for all $ \Kmin $. However, a rigorous proof of this claim is still lacking, and seems to be a difficult problem. Similarly, the tightness of our condition on $ p_0 $ and the tradeoff between $ p_0 $ and $ \tau $ is also unclear in this regime.

\section{Proofs}
\label{sec:proofs}

In this section, we prove Theorems~\ref{thm:disamin} and~\ref{thm:partial}. The proof of Theorem~\ref{thm:lowerbounds} is deferred to Appendix~\ref{sec:proof_lower}.

\subsection{Proof of Theorem \ref{thm:disamin}}
\label{sec:valid}

We first prove Theorem~\ref{thm:disamin}, which says that if the optimization problem~\eqref{eq:partial_conv_prog} produces a valid matrix, i.e., one that corresponds to a clustering of the nodes, then this is the disagreement minimizing clustering. Consider the following non-convex optimization problem
\begin{equation}
\label{eq:nonconv_prog}
\begin{aligned}
  \min_{\B,\K} & \quad \lambda \left\Vert\B\right\Vert_1 +  \left\Vert\K\right\Vert_*  \\
  \text{s.t.} & \quad \PP_{\OmegaO}(\B+\K) = \PP_{\OmegaO}(\I+\A),\\
               & \quad \K\textrm{ is valid,}
\end{aligned}
\end{equation}
and let $(\B,\K)$ be any feasible solution. Since $\K$ represents a
valid clustering, it is positive semidefinite and has all ones along
its diagonal. Therefore, it obeys $\left\| \K
\right\| _{\ast}=\textrm{trace}(\K)=n$. On the other hand, because
both $\K-\I$ and $\A$ are adjacency matrices, the entries of
$\B=\I+\A-\K$ in $ \OmegaO $ must be equal to $-1$, $1$ or $0$ (i.e., it is a
disagreement matrix). Clearly any optimal $ \B $ must have zeros at the entries in $ \OmegaO^c $. Hence $\left\| \B \right\|_{1}=\left\| \PP_{\OmegaO}(\B)
\right\|_{0}$ when $\K$ is valid. We thus conclude that the above optimization problem~\eqref{eq:nonconv_prog}
is equivalent to minimizing $\Vert\PP_{\OmegaO}(\B)\Vert_0$ subject to the same constraints.
This is exactly the minimization of the number of disagreements on
the observed edges. Now notice that~\eqref{eq:partial_conv_prog} is a relaxed version of~\eqref{eq:nonconv_prog}.
Therefore, if the optimal solution of~\eqref{eq:partial_conv_prog} is valid and thus feasible to~\eqref{eq:nonconv_prog},
then it is also optimal to~\eqref{eq:nonconv_prog}, the disagreement minimization problem.

\subsection{Proof of Theorem \ref{thm:partial} }

We now turn to the proof of Theorem~\ref{thm:partial}, which provides guarantees for when the convex program~(\ref{eq:partial_conv_prog}) recovers the true clustering $ (\B^*,\K^*) $.

\subsubsection{Proof Outline and Preliminaries}
\label{sec:outline}

We overview the main steps in the proof of Theorem~\ref{thm:partial}; details are provided in Sections~\ref{sec:step1}--\ref{sec:step3} to follow. We would like to show that the pair $ (\B^*,\K^*) $ corresponding to the true clustering is the unique optimal solution to our convex program~(\ref{eq:partial_conv_prog}). This involves the following three steps.

{Step 1:} We show that it suffices to consider an equivalent model for the observation and disagreements. This model is easier to handle, especially when the observation probability and density gap are vanishingly small, which is the regime of interest in this paper.

{Step 2:} We write down the sub-gradient based first-order sufficient conditions that need to be satisfied for $ (\B^*,\K^*) $ to be the unique optimum of (\ref{eq:partial_conv_prog}). In our case, this involves showing the existence of a matrix $\W$---the {\em dual certificate}---that satisfies certain properties. This step is technical---requiring us to deal with the intricacies of sub-gradients since our convex function is not smooth---but otherwise standard. Luckily for us, this has been done previously \citep{cspw,candeswrightma,xiaodongli}.

{Step 3:} Using the assumptions made on the true clustering $ \K^*  $ and its disagreements $\B^*$, we construct a candidate dual certificate $\W$ that meets the requirements in step 2, and thus certify $(\B^*,\K^*)$ as being the unique optimum. 

The crucial Step 3 is where we go beyond the existing literature on matrix splitting \citep{cspw,candeswrightma,xiaodongli}. These results assume the observation probability and/or density gap is at least a constant, and hence do not apply to our setting. Here we provide a refined analysis, which leads to better performance guarantees than those that could be obtained via a direct application of existing sparse and low-rank matrix splitting results. 

Next, we introduce some notations used in the rest of the proof of the theorem. The following definitions related to $\K^*$ are standard. By symmetry, the SVD of $\K^*$ has the form $\U\mathbf{\Sigma}\U^{T}$, where $ \U\in\mathbb{R}^{n\times r} $ contains the singular vectors of $ \K^* $. We define the subspace $$\T \triangleq \left\{\U\X^{T}+\Y\U^{T}:\X,\Y\in\R^{n\times r}\right\},$$ which is spanned of all matrices that share either the same column space or the same row space as $\K^*$. For any matrix $\M\in\R^{n\times n}$, its {orthogonal projection} to the space $\T$ is given by $\PP_{\T}\left(\M\right)=\U\U^{T}\M+\M\U\U^{T}-\U\U^{T}\M\U\U^{T}$. The projection onto $\T^{\perp}$, the complement orthogonal space of $\T$, is given by $\PP_{\T^{\perp}}(\M) = \M - P_\T(\M)$.

The following definitions are related to $\B^*$ and partial observations. Let $\Omega^*=\{(i,j):  b^{*}_{i,j} \neq 0 \}$ be the set of matrix entries corresponding to the disagreements. Recall that $\OmegaO$ is the set of observed entries. For any matrix $ \M $ and entry set $ \Omega_0$, we let $\PP_{\Omega_0}\left(\M\right)\in\R^{n\times n}$ be the matrix obtained from $\M$ by setting all entries not in the set $\Omega_0$ to zero. We write $ \Omega_0 \sim \textrm{Ber}_0(p) $ if the entry set $ \Omega_0 $ does not contain the diagonal entries, and each pair $ (i,j) $ and $ (j,i) $ ($ i\neq j $) is contained in $ \Omega_0 $ with probability $ p $, independent all others; $ \Omega_0 \sim \textrm{Ber}_1(p) $ is defined similarly except that $ \Omega_0 $ contains all the diagonal entries. Under our partially observed planted partition model, we have $ \OmegaO \sim \textrm{Ber}_1 (p_0) $ and $ \Omega^* \sim \textrm{Ber}_0 (\tau) $.

Several matrix norms are used. $\|\M\|$ and $\|\M\|_F$ represent the spectral and Frobenius norms of the matrix $\M$, respectively, and $\|\M\|_\infty\triangleq \max_{i,j}|m_{i,j}|$ is the matrix infinity norm.

\subsubsection{Step 1: Equivalent Model for Observations and Disagreements}\label{sec:step1}
It is easy show that increasing $ p $ or decreasing $ q $ can only make the probability of success higher, so without loss of generality we assume $ 1-p=q=\tau $. Observe that the probability of success is completely determined by the distribution of $ (\OmegaO, \B^*) $ under the planted partition model with partial observations. The first step is to show that it suffices to consider an equivalent model for generating $(\OmegaO, \B^*)$, which results in the same distribution but is easier to handle. This is in the same spirit as \citet[][Theorems 2.2 and 2.3]{candeswrightma} and \citet[][Section 4.1]{xiaodongli}. In particular, we consider the following procedure:
\begin{enumerate}
\item Let $\Gamma\sim \textrm{Ber}_1\left(p_0(1-2\tau)\right)$, and $\Omega\sim \textrm{Ber}_0\left(\frac{2p_0\tau}{1-p_0+2p_0\tau}\right)$. Let $\OmegaO=\Gamma\cup\Omega$. 
\item Let $\S$ be a symmetric random matrix whose upper-triangular entries are independent and satisfy $ \P(s_{i,j}=1) = \P(s_{i,j}=-1) = \frac{1}{2}$.
\item Define $\Omega'\subseteq\Omega$ as $\Omega'=\left\{ (i,j): (i,j)\in\Omega,s_{i,j}=1-2k^*_{i,j}\right\} $. In other words, $ \Omega' $ is the entries of $ \S $ whose signs are consistent with a disagreement matrix.
\item Define $\Omega^*=\Omega'\backslash\Gamma$, and $\tilde{\Gamma}=\OmegaO\backslash\Omega^*$.
\item Let $\B^*=\PP_{\Omega^*}(\S).$
\end{enumerate}
It is easy to verify that $(\OmegaO,\B^*$)
has the same distribution as in the original model. In particular, we have $\P[(i,j)\in\OmegaO]=p_0$,
$\P[(i,j)\in \Omega^*,(i,j)\in\OmegaO]=p_0\tau$ and $\P[(i,j)\in \Omega^*,(i,j)\notin\OmegaO]=0$, and observe that given $ \K^* $, $ \B^* $ is completely determined by its support $ \Omega^* $.

The advantage of the above model is that $ \Gamma $ and $ \Omega $ are independent of each other, and $ \S $ has random signed entries. This facilitates the construction of the dual certificate, especially in the regime of vanishing $ p_0 $ and $ \left( \frac{1}{2} - \tau \right) $ considered in this paper. We use this equivalent model in the rest of the proof.

\subsubsection{Step 2: Sufficient Conditions for Optimality}

We state the first-order conditions that guarantee $(\B^*,\K^*)$ to be the unique optimum of \eqref{eq:partial_conv_prog} with high probability. Here and henceforth, by \emph{with high probability} we mean with probability at least $1-cn^{-10}$ for some universal constant $c>0$. The following lemma follows from Theorem 4.4 in~\citet{xiaodongli} and the discussion thereafter.
\begin{lemma}
[Optimality Condition]\label{lem:Suff_Optimality_Partial}
\noindent Suppose $\left\Vert \frac{1}{(1-2\tau)p_0}\PP_\T \PP_\Gamma \PP_\T -\PP_\T \right\Vert \le \frac{1}{2}$. Then $(\B^*, \K^*)$ is the unique optimal solution to \eqref{eq:partial_conv_prog} with high probability if there exists $\W \in \R^{n\times n}$ such that
\begin{enumerate}
\item $\left\Vert \PP_{\T}(\W+\lambda \PP_{\Omega}\S-\U\U^{\top})\right\Vert _{F}\le\frac{\llambda}{n^{2}}$,
\item $\left\Vert \PP_{\T^{\bot}}(\W+\llambda \PP_{\Omega}\S)\right\Vert \le\frac{1}{4}$,
\item $\PP_{\Gamma^{c}}(\W)=0$,
\item $\left\Vert \PP_{\Gamma}(\W)\right\Vert _{\infty}\le\frac{\llambda}{4}$.
\end{enumerate}
\end{lemma}
Lemma \ref{lem:OP} in the appendix guarantees that the condition $\left\Vert \frac{1}{(1-2\tau) p_0}\PP_\T \PP_\Gamma \PP_\T -\PP_\T \right\Vert \le \frac{1}{2}$ is satisfied with high probability under the assumption of Theorem \ref{thm:partial}. It remains to show the existence of a desired dual certificate $ \W $ which satisfies the four conditions in Lemma~\ref{lem:Suff_Optimality_Partial} with high probability.

\subsubsection{Step 3: Dual Certificate Construction}\label{sec:step3}

We use a variant of the so-called \emph{golfing scheme}~\citep{candeswrightma,Gross} to construct $\W$. Our application of golfing scheme, as well as its analysis, is different from previous work and leads to stronger guarantees. In particular, we go beyond existing results by allowing the fraction of observed entries and the density gap to be vanishing.

By definition in Section~\ref{sec:step1}, $\Gamma$ obeys $\Gamma\sim\textrm{Ber}_1(p_0(1-2\tau))$. Observe that $\Gamma$ may be considered to be generated by
$\Gamma = \bigcup_{1\le k \le k_0}\Gammak$, where the sets $\Gammak \sim\textrm{Ber}_1(t)$ are independent; here the parameter $t$ obeys $p_0(1-2\tau)=1-(1-t)^{k_0}$, and $k_0$ is chosen to be $\left\lceil 5\log n\right\rceil $. This implies $t \ge p_{0}(1-2\tau)/k_0\ge C_{0}\frac{n\log
n}{K_{\min}^2}$ for some constant $C_{0}$, with the last inequality holds under the assumption of Theorem \ref{thm:partial}. For any random entry set $\Omega_0\sim\textrm{Ber}_1(\rho)$,
define the operator $\RR_{\Omega_{0}}:\mathbb{R}^{n\times n} \mapsto \mathbb{R}^{n\times n}$ by
\begin{align*}
  \RR_{\Omega_{0}} (\M) 
 =  \sum_{i=1}^{n}m_{i,i}e_{i}e_{i}^{T} + \rho^{-1} \sum_{1\le i<j\le
n}\delta_{ij}^{}m_{i,j}\left(e_{i}e_{j}^{T}+e_{j}e_{i}^{T}\right),
\end{align*}
where $ \delta_{ij} $ is the indicator random variable with $\delta_{ij}^{}=1$ if $(i,j) \in \Omega_0$ and 0 otherwise,
and $e_i$ is the $i$-th standard basis in $ \mathbb{R}^{n\times n} $, i.e., the 
column vector with $1$ in its $i$-th entry and $0$ elsewhere.

We now define our dual certificate. Let $\W=\W_{k_0}$,
where $\W_{k_0}$ is defined recursively by setting
$\W_0 = 0$ and for all $k=1,2,\ldots,k_0$,
\begin{equation*}
 \W_{k}
 =  \W_{k-1}+\RR_{\Gamma_{k}}\PP_{\T}\left( \U\U^T-\llambda  \PP_{\T}(\EE)-\W_{k-1} \right).
\end{equation*}

Clearly the equality condition in Lemma~\ref{lem:Suff_Optimality_Partial} is satisfied. It remains to show that $\W$ also satisfies the inequality conditions with high probability. The proof makes use of the technical Lemmas~\ref{lem:OP}--\ref{lem:PTE_inf} given in the appendix. For convenience of notation, we define the quantity $\Delta_{k} = \U\U^T
-\llambda \PP_{\T}\left(\EE\right) -\PP_{\T}(\W_{k})$, and use the notation
\[
\prod_{i=1}^{k} (\PP_{\T}-\PP_{\T}\RR_{\Gammai}\PP_{\T})
  = (\PP_{\T}-\PP_{\T}\RR_{\Gammak}\PP_{\T}) \cdots (\PP_{\T}-\PP_{\T}\RR_{\Gamma_{1}}\PP_{\T}),
\]
where the order of multiplication is important. Observe that by
construction of $\W$, we have
\begin{align}
 \Delta_{k}  = & \prod_{i=1}^{k} (\PP_{\T}-\PP_{\T}\RR_{\Gammai}\PP_{\T}) (\U\U^T - \llambda \PP_\T\EE), k=1,\ldots,k_0, \label{eq:stepError}\\
 \W_{k_0} = & \sum_{k=1}^{k_0}\RR_{\Gamma_{k}}\Delta_{k-1}.
\label{eq:SumStepError}
\end{align}
We are ready to prove that $ \W $ satisfies inequalities 1, 2 and 4 in Lemma~\ref{lem:Suff_Optimality_Partial}.

\textit{Inequality 1:} 
Thanks to~\eqref{eq:stepError}, we have the following geometric convergence :
\begin{align*}
\left\Vert \PP_{\T}(\W+\lambda \PP_{\Omega}\S-\U\U^{\top})\right\Vert_{F}  &=   \left\Vert \Delta_{k_0}\right\Vert_F \\
 & \le  \left(\prod_{k=1}^{k_0}\left\Vert \PP_{\T}-\PP_{\T}\RR_{\Gamma_{k}}\PP_{\T}\right\Vert \right)\left\Vert \U\U^T - \llambda \PP_{\T} \EE\right\Vert _{F}\\
 & \overset{(a)}{\le}  e^{-k_0} ( \left\Vert \U\U^T \right \Vert_F + \llambda \left\Vert \PP_{\T}\EE\right\Vert _{F})\\
 &\overset{(b)}{\le} n^{-5} ( n +  \llambda \cdot n)
 \le  (1+ \llambda) n^{-4}\,\,
  \overset{(c)}{\le} \,\, \frac{1}{2n^{2}} \llambda.
\end{align*}
Here, the inequality (a) follows Lemma \ref{lem:popr} with $\epsilon_1 = e^{-1}$, (b) follows from our choices of $\lambda$ and $k_0$ and the fact that $\left\Vert
\PP_{\T} \EE\right\Vert _{F}\le\left\Vert \EE\right\Vert _{F}\le n$,
and (c) holds under the assumption $\llambda \ge \frac{1}{32\sqrt{n}}$ in the theorem.

\textit{Inequality 4:} 
We have
\begin{align}
\left\Vert \PP_{\Gamma}(\W) \right\Vert _{\infty}
= \left\Vert \PP_{\Gamma}(\W_{k_0}) \right\Vert _{\infty}  \le   \sum_{k=1}^{k_0}\left\Vert \RR_{\Gamma_{i}}\Delta_{i-1}\right\Vert _{\infty}\nonumber
\le t^{-1}\sum_{k=1}^{k_0}\left\Vert \Delta_{k-1}\right\Vert _{\infty},\nonumber
\end{align}
where the first inequality follows from~\eqref{eq:SumStepError} and the triangle inequality. We proceed to obtain
\begin{align}
 \sum_{k=1}^{k_0}\left\Vert \Delta_{k-1}\right\Vert _{\infty}
 & \overset{(a)}{=}  \sum_{k=1}^{k_0}\left\Vert \prod_{i=1}^{k-1} (\PP_{\T}-\PP_{\T}\RR_{\Gammai}\PP_{\T}) (\U\U^T - \llambda \PP_\T\EE)\right\Vert _{\infty} \nonumber \\
 & \overset{(b)}{\le}  \sum_{k=1}^{k_0} \left(\frac{1}{2}\right)^k \left\Vert  \U\U^T - \llambda \PP_\T\EE\right\Vert _{\infty} \nonumber \\
 & \overset{(c)}{\le}  \frac{1}{\Kmin} + \llambda \sqrt{\frac{p_0 n\log n}{\Kmin^2}},
 \label{eq:sum_delta_inf}
\end{align} 
where (a) follows from~\eqref{eq:stepError}, (b) follows from Lemma~\ref{lem:INF} and (c) follows from Lemma~\ref{lem:PTE_inf}. It follows that
\begin{align*}
\left\Vert \PP_{\Gamma}(\W) \right\Vert _{\infty}
& \le  \frac{1}{t} \left( \frac{1}{\Kmin} + \frac{n\log n}{\Kmin^2}\llambda \right)\\
& \le  \frac{k_0}{p_0(1-2\tau)}  \left( \frac{1}{\Kmin} + \llambda \sqrt{\frac{p_0 n\log n}{\Kmin^2}} \right)
 \le  \frac{1}{4}\llambda,
\end{align*}
where the last inequality holds under the assumption of Theorem \ref{thm:partial}

\textit{Inequality 2:} 
Observe that by the triangle inequality, we have
$$\left\Vert \PP_{\T^{\perp}} (\W + \llambda \EE) \right\Vert \le\llambda
 \left\Vert \PP_{\T^{\perp}} (\EE) \right\Vert
 +\left\Vert \PP_{\T^{\perp}}(\W_{k_0}) \right\Vert .$$
For the first term, standard results on the norm of a
matrix with i.i.d. entries (e.g., see~\citealt{Versh}) give
\begin{equation*}
\llambda\left\Vert \PP_{\T^{\perp}} (\EE) \right\Vert 
\le \llambda\left\Vert \EE \right\Vert 
\le \frac{1}{32\sqrt{p_0 n}} \cdot 4\sqrt{\frac{2p_0\tau n}{1-p_0+2p_0\tau}}
\le\frac{1}{8}
\end{equation*}
It remains to show that the second term is bounded by $\frac{1}{8}$. To this end, we observe that
\begin{align}
\left\Vert \PP_{\T^{\perp}} (\W_{k_0}) \right\Vert
 & \overset{(a)}{=}  \sum_{k=1}^{k_0}\left\Vert \PP_{\T^{\perp}}\left(\RR_{\Gamma_{k}}\Delta_{k-1}-\Delta_{k-1}\right)\right\Vert \nonumber\\
 & \le  \sum_{k=1}^{k_0}\left\Vert \left(\RR_{\Gamma_{k}}-\mathcal{I}\right)\Delta_{k-1}\right\Vert \nonumber \\
 & \overset{(b)}{\le}  C \sqrt{\frac{n\log n}{t}}\sum_{k=1}^{k_0} \left\Vert \Delta_{k-1}\right\Vert_\infty \nonumber \\
 & \overset{(c)}{\le}  C \sqrt{\frac{k_0 n\log n}{p_0 (1-2\tau)}} \left(\frac{1}{\Kmin} + \llambda \sqrt{\frac{p_0 n\log n}{\Kmin^2}}\right) \nonumber
  \le \frac{1}{8}, \nonumber
\end{align}
where in (a) we use~\eqref{eq:SumStepError} and the fact that $\Delta_{k}\in\T$, (b) follows from Lemma~\ref{lem:OP_INF}, and (c) follows from~\eqref{eq:sum_delta_inf}.
This completes the proof of Theorem~\ref{thm:partial}.

\section{Experimental Results}
\label{sec:exper}

We explore via simulation the performance of our algorithm as a function of the values of the model parameters $ (n,\Kmin,p_0, \tau) $. We see that the performance matches well with the theory. 

In the experiment, each test case is constructed by generating a graph with $ n $ nodes divided into clusters of equal size~$\Kmin$, and then placing a disagreement on each pair of node with probability~$\tau$ independently. Each node pair is then observed with probability $p_0$. We then run Algorithm~\ref{alg1}, where the optimization problem~\eqref{eq:partial_conv_prog} is solved using the fast algorithm in \citet{LinALM}.. We check if the algorithm successfully outputs a solution that equals to the underlying true clusters.
In the first set of experiments, we fix $ \tau = 0.2 $ and $\Kmin=n/4$ and vary $p_0$ and $ n $.  For each $ (p_0,n) $, we repeat the experiment for $5$ times and plot the probability of success in the left pane of Figure~\ref{fig:rho_n}. 

\begin{figure}[]
\centering
\includegraphics[width=0.5\columnwidth]{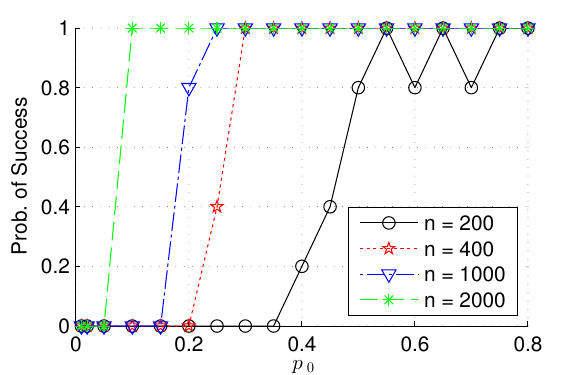}\includegraphics[width=0.5\columnwidth]{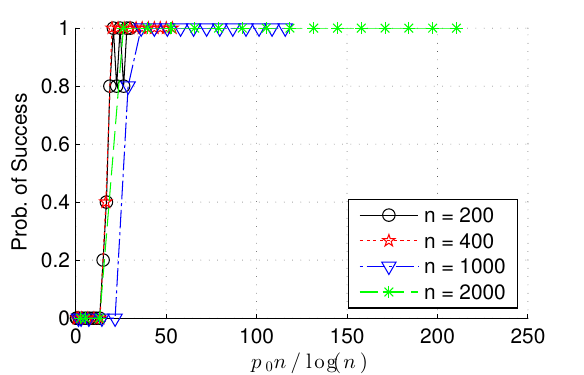}
\caption{\small Simulation results verifying the performance of our algorithm as a function of the observation probability $ p_0 $ and the graph size $ n $. The left pane shows the probability of successful recovery under different $ p_0 $ and $ n $ with fixed $ \tau=0.2 $ and $ \Kmin = n/4 $; each point is an average over $ 5 $ trials. After proper rescaling of the x-axis, the curves align as shown in the right pane, indicating a good match with our theoretical results. \label{fig:rho_n}}
\end{figure}

One observes that our algorithm has better performance with larger $p_0$ and $ n $, and the success probability exhibits a phase transition. Theorem \ref{thm:partial} predicts that, with $ \tau $ fixed and $ \Kmin = n/4 $, the transition occurs at $ p_0 \propto \frac{n\log^2 n}{\Kmin^2} \propto \frac{\log^2 n}{n}$; in particular, if we plot the success probability against the control parameter $ \frac{p_0 n}{\log^2 n} $, all curves should align with each other. Indeed, this is precisely what we see in the right pane of Figure~\ref{fig:rho_n} where we use $\frac{p_0 n}{\log n} $ as the control parameter. This shows that Theorem \ref{thm:partial} gives the correct scaling between $ p_0 $ and $ n $ up to an extra log factor.

In a similar fashion, we run another three sets of experiments with the following settings: (1) $ n=1000$ and $\tau=0.2 $ with varying $ (p_0, \Kmin)$; (2) $ \Kmin = n/4 $ and $ p_0 =0.2 $ with varying $(\tau, n)$; (3) $ n=1000 $ and $ p_0 = 0.6 $ with varying $ (\tau, \Kmin) $.  The results are shown in Figures~\ref{fig:rho_K}, \ref{fig:tau_n} and~\ref{fig:tau_K}; note that each $ x $-axis corresponds to a control parameter chosen according to the scaling predicted by Theorem \ref{thm:partial}. Again we observe that all the curves roughly align, indicating a good match with the theory. In particular, by comparing Figures~\ref{fig:rho_n} and~\ref{fig:tau_n} (or Figures~\ref{fig:rho_K} and~\ref{fig:tau_K}), one verifies the quadratic tradeoff between observations and disagreements (i.e., $ p_0 $ vs. $ 1-2\tau $) as predicted by Theorem~\ref{thm:partial}.

\begin{figure}[t]
\centering
\includegraphics[width=0.5\columnwidth]{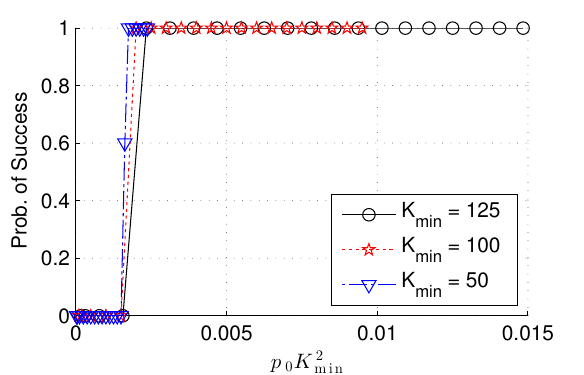}
\caption{\small Simulation results verifying the performance of our algorithm as a function of the observation probability $ p_0 $ and the cluster size $ \Kmin $, with $ n=1000$ and $ \tau=0.2$ fixed.\label{fig:rho_K}}
\end{figure}

\begin{figure}[t]
\centering
\includegraphics[width=0.5\columnwidth]{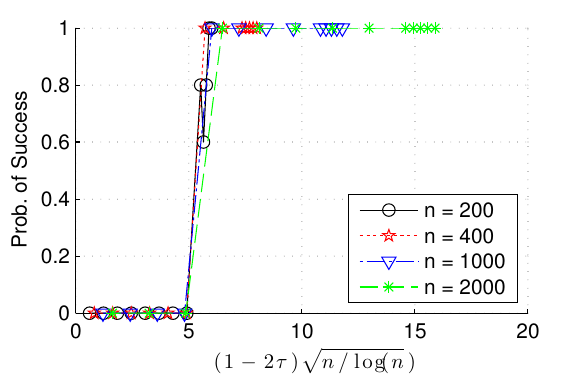}
\caption{\small Simulation results verifying the performance of our algorithm as a function of the disagreement probability $ \tau $ and the graph size $ n $, with $ p_0=0.2$ and $ \Kmin=n/4$ fixed.\label{fig:tau_n}}
\end{figure}

\begin{figure}[t]
\centering
\includegraphics[width=0.5\columnwidth]{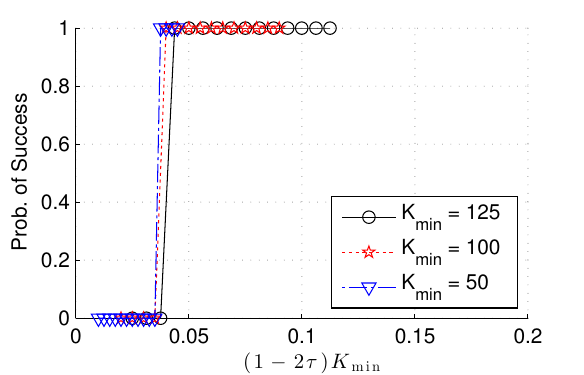}
\caption{\small Simulation results verifying the performance of our algorithm as a function of the disagreement probability $ \tau $ and the cluster size $ \Kmin $, with $ n=1000$ and $ p_0=0.6$ fixed.\label{fig:tau_K}}
\end{figure}

Finally, we compare the performance of our method with spectral clustering, a popular method for graph clustering. For spectral clustering, we first impute the missing entries of the adjacency matrix with either zeros or random $ 1/0 $'s. We then compute the first $ k $ principal components of the adjacency matrix, and run $ k $-means clustering on the principal components \citep{von2007spectral}; here we set $ k $ equal to the  number of clusters. The adjacency matrix is generated in the same fashion as before using the parameters $ n=2000 $, $ \Kmin=200 $ and $ \tau=0.1 $. We vary the observation probability $ p_0 $ and plot the success probability in Figure \ref{fig:compare}. It can be observed that our method outperforms spectral clustering with both imputation schemes; in particular, it requires fewer observations.

\begin{figure}[t]
\centering
\includegraphics[width=0.5\columnwidth]{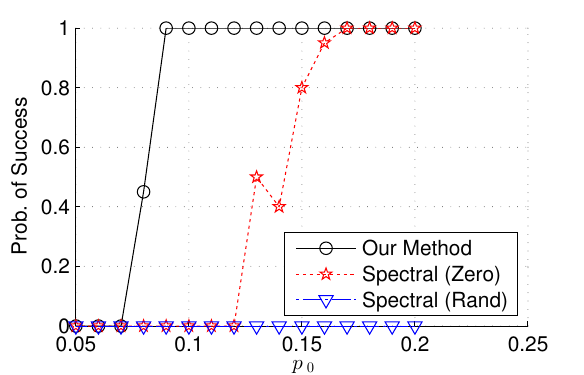}
\caption{\small Comparison of our method and spectral clustering under different observation probabilities $ p_0 $, with $ n=2000 $, $ \Kmin=200 $ and $ \tau=0.1 $. For spectral clustering, two imputation schemes are considered: (a) Spectral (Zero), where the missing entries are imputed with zeros, and (b) Spectral (Rand), where they are imputed with $ 0/1 $ random variables with symmetric probabilities. The result shows that our method recovers the underlying clusters with fewer observations. \label{fig:compare}}
\end{figure}

\section{Conclusion}\label{sec:conclusion}

We proposed a convex optimization formulation, based on a reduction to decomposing low-rank and sparse matrices, to address the problem of clustering partially observed graphs. We showed that under a wide range of parameters of the planted partition model with partial observations, our method is guaranteed to find the optimal (disagreement-minimizing) clustering. In particular, our method succeeds under higher levels of noise and/or missing observations than existing methods in this setting. The effectiveness of the proposed method and the scaling of the theoretical results are validated by simulation studies. 

This work is motivated by graph clustering applications where obtaining similarity data is expensive and it is desirable to use as few observations as possible. As such, potential directions for future work include considering different sampling schemes such as active sampling, as well as dealing with sparse graphs with very few connections.

\acks{S. Sanghavi would like to acknowledge DTRA grant HDTRA1-13-1-0024 and NSF grants 1302435, 1320175 and 0954059. H. Xu is partially supported by the Ministry of Education of Singapore through AcRF Tier Two grant R-265-000-443-112. The authors are grateful to the anonymous reviewers for their thorough reviews of this work and valuable suggestions on improving the manuscript.}

\appendix

\newcommand{\Gammad}{\Gamma_{\textrm{d}}}
\newcommand{\Gammar}{\Gamma_{\textrm{r}}}
\newcommand{\Omegad}{\Omega_{\textrm{d}}}
\newcommand{\Omegar}{\Omega_{\textrm{r}}}

\newcommand{\Omegatwo}{\Omega^{(2)}}
\newcommand{\Omegako}{\Omega^{(k_0)}}

\section{Technical Lemmas}
In this section, we provide several auxiliary lemmas required in the proof of Theorem \ref{thm:partial}. We will make use of the non-commutative Bernstein
inequality. The following version is given by \citet{Tropp}. 

\begin{lemma}
[\citealp{Tropp}]\label{lem:bernstein} Consider a finite
sequence $\{\M_{i}\}$ of independent, random $n\times n$
matrices that satisfy the assumption $\mathbb{E} \M_{i}=0$ and
$\left\Vert \M_{i}\right\Vert \le D$ almost surely. Let
\[
\sigma^{2}=\max\left\{ \left\Vert
\sum_{i}\mathbb{E} \left[\M_{i}\M_{i}^{\top}\right]\right\Vert
,  \left\Vert
\sum_{i}\mathbb{E} \left[\M_{i}^{\top}\M_{i}\right]\right\Vert
\right\}.
\]
Then for all $\theta>0$ we have
\begin{align}
\mathbb{P}\left[\left\Vert \sum \M_{i}\right\Vert \ge \theta\right]
 & \le   2n\exp\left(-\frac{\theta^{2}}{2\sigma^{2}+2D\theta/3}\right). \nonumber \\
 & \le  \begin{cases}
            2n\exp\left(-\frac{3\theta^{2}}{8\sigma^{2}}\right), & \quad\textrm{for }\theta\le\frac{\sigma^{2}}{D};\\
            2n\exp\left(-\frac{3\theta}{8D}\right), & \quad\textrm{for
            }\theta\ge\frac{\sigma^{2}}{D}.
         \end{cases}
\label{eq:bernstein}
\end{align}
\end{lemma}
\begin{remark}
When $n=1$, this becomes the standard two-sided Bernstein inequality.
\end{remark}

We will also make use of the following estimate, which follows from
the structure of $\U$.
\begin{equation*}
 \left\Vert \PP_{\T}(e_{i}e_{j}^{\top})\right\Vert _{F}^{2}
 = \left\Vert \U\U^T e_i \right\Vert^2 + \left\Vert \U\U^T e_j \right\Vert^2 - \left\Vert \U\U^T e_i
 \right\Vert^2 \left\Vert \U\U^T e_j \right\Vert^2
 \le  \frac{2n}{\Kmin^{2}},\quad\forall 1 \le i,j \le n.
\end{equation*}

The first auxiliary lemma controls the operator norm of certain random operators. A similar result was given in~\citet[][Theorem 4.1]{candeswrightma}. Our proof is different from theirs.
\begin{lemma}
Suppose $\Omega_0$ is a set of entries obeying $\Omega_0\sim\textrm{Ber}_1(\rho)$. Consider the operator $P_\T - P_\T \RR_{\Omega_0} P_\T$. For some constant $C_0>0$, we have
  \begin{equation*}
      \left\| \PP_\T - \PP_\T \RR_{\Omega_0} \PP_\T \right\| < \epsilon_1
  \end{equation*}
with high probability provided that $\rho \ge C_0\frac{n \log
n}{\epsilon_1^2 \Kmin^2}$ and $\epsilon_1 \le 1$.
  \label{lem:popr}
  \label{lem:OP}
\end{lemma}

\begin{proof}
For each $(i,j)$, define the indicator random variable $\delta_{ij}=\mathbf{1}_{\{(i,j)\in\Omega_0\}}$. We observe that for any matrix $\M\in \T$,
\begin{align*}
 \left(\PP_{\T}\RR_{\Omega_0}\PP_{\T}-\PP_{\T}\right) \M
 & =\sum_{1\le i<j\le n}\mathcal{S}_{ij}(\M)\\
 & \triangleq \sum_{1\le i<j\le n}\left(\rho^{-1}\delta_{ij}-1\right)\left\langle \PP_{\T}(e_{i}e_{j}^{\top}),\; \M\right\rangle \PP_{\T}(e_{i}e_{j}^{\top}+e_{j}e_{i}^{\top}).
\end{align*}
Here $\mathcal{S}_{ij}: \mathbb{R}^{n\times
n}\mapsto\mathbb{R}^{n\times n}$ is a linear self-adjoint operator
with $\mathbb{E}\left[\mathcal{S}_{ij}\right]=0$. Using the fact that $\PP_{\T} (e_i e_j^{\top}) = \left( \PP_{\T} (e_j
e_i^{\top}) \right) ^{\top}$ and $\M$ is symmetric, we obtain the
bounds
\begin{align*}
\left\Vert \mathcal{S}_{ij}\right\Vert  
 & \le  \rho^{-1}\left\Vert \PP_{\T}(e_{i}e_{j}^{\top})\right\Vert _{F}\left\Vert \PP_{\T}(e_{i}e_{j}^{\top}+e_{j}e_{i}^{\top})\right\Vert _{F}\\
 & \le  \rho^{-1}\cdot2\left\Vert \PP_{\T}(e_{i}e_{j}^{\top})\right\Vert _{F}^{2}\le\frac{4n}{\Kmin^{2}\rho},
\end{align*}
and
\begin{align*}
   & \left\Vert \mathbb{E}\left[\sum_{1\le i<j\le n}\mathcal{S}_{ij}^{2}(\M)\right]\right\Vert _{F}\\
= & \left\Vert \sum_{1\le i<j\le n}\mathbb{E}\left[(\rho^{-1}\delta_{ij}^{(k)}-1)^2\right]\left\langle \PP_{\T}(e_{i}e_{j}^{\top}),\; \M\right\rangle\left\langle \PP_{\T}(e_{i}e_{j}^{\top}+e_{j}e_{i}^{\top}),\; e_{i}e_{j}^{\top}\right\rangle \PP_{\T}(e_{i}e_{j}^{\top}+e_{j}e_{i}^{\top})\right\Vert _{F}\\
= & \left(\rho^{-1}-1\right)\left\Vert \sum_{1\le i<j\le n} 2\left\Vert \PP_{\T}(e_{i}e_{j}^{\top})\right\Vert _{F}^{2} m_{i,j} \PP_{\T}(e_{i}e_{j}^{\top}+e_{j}e_{i}^{\top})\right\Vert _{F}\\
\le & \left(\rho^{-1}-1\right)\left\Vert \sum_{1\le i<j\le n} 2\left\Vert \PP_{\T}(e_{i}e_{j}^{\top})\right\Vert _{F}^{2} m_{i,j} (e_{i}e_{j}^{\top}+e_{j}e_{i}^{\top})\right\Vert _{F}\\
\le & \left(\rho^{-1}-1\right)\frac{4n}{\Kmin^{2}}\left\Vert \sum_{1\le i<j\le n} m_{i,j} (e_{i}e_{j}^{\top}+e_{j}e_{i}^{\top})\right\Vert _{F}\\
= & \left(\rho^{-1}-1\right)\frac{4n}{\Kmin^{2}}\left\Vert \M\right\Vert _{F},
\end{align*}
which means $\left\Vert \mathbb{E}\left[\sum_{1\le i<j\le
n}\mathcal{S}_{ij}^{2}\right]\right\Vert \le\frac{4n}{\Kmin^{2}\rho}$. Applying the first inequality in the Bernstein inequality~\eqref{eq:bernstein} gives
\[
\mathbb{P}\left[\left\Vert \textstyle{\sum_{1\le i<j\le
n}\mathcal{S}_{ij}} \right\Vert
\ge\epsilon_{1}\right]\le2n^{2-2\beta}
\]
provided $\rho\ge\frac{64\beta n\log n} {3\Kmin^{2}\epsilon_{1}^{2}}$ and $\epsilon_{1}<1$.
\end{proof}

The next lemma bounds the spectral norm of certain symmetric random matrices. A related result for non-symmetric matrices appeared in~\citet[][Theorem 6.3]{candesrecht}.
\begin{lemma}
Suppose $\Omega_0$ is a set of entries obeying $\Omega_0\sim\textrm{Ber}_1(\rho)$, and $\M$ is a fixed $n\times n$ symmetric matrix. Then for some constant $C_0>0$, we have
\begin{equation*}
      \left\| (\II - \RR_{\Omega_0}) \M \right\| < \sqrt{ C_0 \frac{n\log n} {\rho}
      } \| \M \|_\infty,
\end{equation*}
with high probability provided that $\rho \ge C_0\frac{\log
n}{n}$.
  \label{lem:popinf}
  \label{lem:OP_INF}
\end{lemma}

\begin{proof}
Define $\delta_{ij}$ as before. Notice that
\begin{equation*}
 \RR_{\Omega_0}(\M)-\M
 = \sum_{i<j}S_{ij}
 \triangleq \sum_{i<j}(\rho^{-1}\delta_{ij}-1)m_{i,j}\left(e_{i}e_{j}^{\top}+e_{j}e_{i}^{\top}\right).
\end{equation*}
Here the symmetric matrix $S_{ij}\in\mathbb{R}^{n\times n}$
satisfies $\mathbb{E}\left[S_{ij}\right]=0$, $\left\Vert
S{}_{ij}\right\Vert \le 2\rho^{-1}\left\Vert \M\right\Vert _{\infty}$
and the bound
\begin{align*}
\left\Vert \mathbb{E}\left[ \textstyle{\sum_{i<j}S_{ij}^2}\right]\right\Vert  
& =  \left(\rho^{-1}-1\right)\left\Vert \sum_{i<j}m_{i,j}^{2}\left(e_{i}e_{i}^{\top}+e_{j}e_{j}^{\top}\right)\right\Vert \\
 & \le  \left(\rho^{-1}-1\right)\left\Vert \textrm{diag}\left(\sum_{j}m_{1,j}^{2},\ldots,\sum_{j}m_{n,j}^{2}\right)\right\Vert \\
 & \le  \left(\rho^{-1}-1\right)n\left\Vert \M\right\Vert_{\infty}^{2}\le 2\rho^{-1}n\left\Vert \M\right\Vert _{\infty}^{2}.
\end{align*}
When $\rho\ge\frac{16\beta\log n}{3n}$, we apply the first inequality in the
Bernstein inequality~\eqref{eq:bernstein} to obtain
\begin{align*}
 \mathbb{P}\left[\left\Vert \textstyle{\sum_{i<j}S_{ij}}\right\Vert \ge \sqrt{ \frac{16 \beta n \log n}{3\rho}} \left\Vert \M\right\Vert _{\infty}\right]
 & \le  2n\exp\left(-\frac{ 3\cdot \frac{16 \beta n\log n}{3\rho} \left\Vert \M\right\Vert _{\infty}^{2}}{8 \cdot \frac{2n}{\rho}\left\Vert \M\right\Vert _{\infty}^{2}}\right)
  \le  2n^{1-\beta}.
\end{align*}
The conclusion follows by choosing a sufficiently large constant $\beta$.
\end{proof}

The third lemma bounds the infinity norm of certain random symmetric matrices. A related result is given in~\citet[][Lemma 3.1]{candeswrightma}.
\begin{lemma}
Suppose $\Omega_0$ is a set of entries obeying $\Omega_0\sim\textrm{Ber}_1(\rho)$, and $\M \in \T$ is a fixed symmetric $n\times n$ matrix. Then for some constant $C_0>0$, we have
  \begin{equation*}
      \| (\PP_{\T} - \PP_{\T} \RR_{\Omega_0} \PP_\T) \M \|_\infty
      < \epsilon_3 \| \M \|_\infty,
  \end{equation*}
with high probability provided that $\rho \ge C_0 \frac{n \log
n}{\epsilon_3^2 \Kmin^2}$ and $\epsilon_3 \le 1$.
  \label{lem:pinf}
  \label{lem:INF}
\end{lemma}

\begin{proof}
Define $\delta_{ij}$ as before. Fix an entry index $(a,b)$. Notice that
\begin{align*}
  \left( \PP_{\T} \RR_{\Omega_0} \PP_\T \M - \PP_{\T} \M\right)_{a,b}
= \sum_{i<j}\xi_{ij}
\triangleq \sum_{i<j}\left\langle (\rho^{-1}\delta_{ij}^{(k)}-1)m_{i,j}\PP_{\T}\left(e_{i}e_{j}^{\top}+e_{j}e_{i}^{\top}\right),\; e_{a}e_{b}^{\top}\right\rangle.
\end{align*}
The random variable $ \xi_{ij} $ satisfies $\mathbb{E}\left[\xi_{ij}\right]=0$ and obeys the bounds
\begin{align*}
 \left|\xi_{ij}\right| 
 \le  2p^{-1}\left\Vert \PP_{\T}(e_{i}e_{j}^{\top})\right\Vert _{F}\left\Vert \PP_{\T}(e_{a}e_{b}^{\top})\right\Vert _{F}\left|m_{i,j}\right|
 \le \frac{4n}{\Kmin^{2}\rho}\left\Vert \M\right\Vert_{\infty}
\end{align*}
and
\begin{align*}
\left|\mathbb{E}\left[\sum_{i<j}\xi_{ij}^{2}\right]\right|
 & =  \left|\sum_{i<j} \mathbb{E}\left[(\rho^{-1}\delta_{ij}^{(k)}-1)^{2}\right] m_{i,j}^{2} \left\langle \PP_{\T}\left(e_{i}e_{j}^{\top}+e_{j}e_{i}^{\top}\right),\; e_{a}e_{b}^{\top}\right\rangle ^{2}\right|\\
 & \le  \left(\rho^{-1}-1\right)\left\Vert \M \right\Vert _{\infty}^{2} \sum_{i<j}\left\langle e_{i}e_{j}^{\top}+e_{j}e_{i}^{\top},\; \PP_{\T}(e_{a}e_{b}^{\top})\right\rangle ^{2}\\
 & \le  2\left(\rho^{-1}-1\right)\left\Vert \M\right\Vert _{\infty}^{2}\left\Vert \PP_{\T}(e_{a}e_{b}^{\top})\right\Vert _{F}^{2}\\
 & \le  2\left(\rho^{-1}-1\right)\frac{2n}{\Kmin^{2}}\left\Vert \M\right\Vert _{\infty}^{2}\\
 & \le  \frac{4n}{\Kmin^{2}\rho}\left\Vert \M\right\Vert_{\infty}^{2}.
\end{align*}
When $\rho\ge\frac{64\beta n\log n}{3\Kmin^{2}\epsilon_{3}^{2}}$ and
$\epsilon_3 \le 1$, we apply the first inequality in the Bernstein inequality~\eqref{eq:bernstein} with $ n=1 $ to obtain
\begin{align*}
  \mathbb{P}\left[\left|\left(\PP_{\T} \RR_{\Omega_0} \PP_\T \M - \PP_{\T} \M\right)_{a,b}\right|\ge\epsilon_{3}\left\Vert \M\right\Vert _{\infty}\right] 
 \le 2\exp\left(-\frac{3\epsilon_{3}^{2}\left\Vert \M\right\Vert _{\infty}^{2}}{8\frac{4n}{\Kmin^{2}\rho}\left\Vert \M\right\Vert _{\infty}^{2}}\right)
 \le 2n^{-2\beta}.
\end{align*}
Applying the union bound then yields
\begin{align*} \mathbb{P}\left[\left\Vert
\PP_{\T} \RR_{\Omega_0} \PP_\T \M - \PP_{\T} \M\right\Vert
_{\infty}\ge\epsilon_{3}\left\Vert \M\right\Vert _{\infty}\right] 
\le  2n^{2-2\beta}.
\end{align*}
\end{proof}

The last lemma bounds the matrix infinity norm of $\PP_{\T}
\EE$ for a $ \pm 1 $ random matrix $ \S$. 

\begin{lemma}
Suppose $ \Omega \sim \textrm{Ber}_0\left(\frac{2p_0\tau}{1-p_0 + 2p_0 \tau}\right) $ and $\S \in \mathbb{R}^{n\times n}$ has i.i.d. symmetric $\pm 1$ entries . Under the assumption of Theorem \ref{thm:partial}, for some constant $C_0$, we have with high probability
  \begin{equation*}
      \left\| \PP_{\T} \EE \right\|_\infty \le  C_0\sqrt{\frac{ p_0 n \log n}{\Kmin^2}}.
  \end{equation*}
  \label{lem:pte}
  \label{lem:PTE_inf}
\end{lemma}

\begin{proof}
By triangle inequality, we have 
\begin{equation}
\left\Vert \PP_{\T} \EE\right\Vert _{\infty} 
\le \left\Vert \U\U^{T}\EE\right\Vert _{\infty}
 +\left\Vert \EE \U\U^{T}\right\Vert _{\infty}
  +\left\Vert \U\U^{T} \EE \U\U^{T}\right\Vert _{\infty},
\nonumber
\end{equation}
so it suffices to show that each of these three terms are bounded by $ C \sqrt{\frac{p_0 n \log n}{\Kmin^2}} $ w.h.p. 
for some constant $C $. Under the assumption on $ \Omega $ and $ \S $ in the lemma statement, each pair of symmetric entries of $\EE$  equals $\pm 1$ 
with probability $\rho \triangleq \frac{p_0\tau}{1-p_0 + 2p_0 \tau}$ and $0$ otherwise; notice that $ \rho \le \frac{p_0}{2} $ since $ \tau \le \frac{1}{2} $. 
Let  $\left( s^{(i)} \right)^T$ be the $i$th row of $\U\U^{T}$. From the structure of $\U$, we know that for all $ i $ and $ j $,
\begin{align*}
\left|s_{j}^{(i)}\right|  \le
\frac{1}{\Kmin},
\end{align*}
and for all $ i $,
\begin{align*}
\sum_{j=1}^{n}\left(s_{j}^{(i)}\right)^{2}  \le  \frac{1}{\Kmin}.
\end{align*}
We now bound $\left\Vert \U\U^{T}\EE\right\Vert _{\infty}$. For
simplicity, we focus on the $(1,1)$ entry of $\U\U^{T}\EE$ and
denote this random variable as $X$. We may write $ X $ as
$X=\sum_{i}s_{i}^{(1)} \left( \EE \right)_{i,1}$ , for which we have
\begin{align*}
\mathbb{E}\left[s_{i}^{(1)}\left(\EE \right)_{i,1}\right] & =  0,\\
\left|s_{i}^{(1)}\left(\EE \right)_{i,1}\right| 
  & \le \left|s_{i}^{(1)}\right|\le \frac{1}{\Kmin},\quad a.s.\\
\textrm{Var}\left(X\right) 
  & =  \sum_{i:(i,1)\in\Omega}(s_{i}^{(1)})^{2} \cdot 2\rho \le \frac{p_{0}}{\Kmin}.
\end{align*}
Applying the standard Bernstein inequality then gives
\begin{align*}
\mathbb{P}\left[\left|X\right| > C \sqrt{\frac{p_0 n \log n}{ \Kmin^2}}\right]
  & \le  2\exp\left[ -\left( C^2 \frac{p_0 n\log n}{\Kmin^2}\right) / \left( 2\frac{p_0 }{\Kmin} +  \frac{2C \sqrt{p_0 n \log n}}{3\Kmin^2}\right) \right].
\end{align*}
Under the assumption of Theorem~\ref{thm:partial}, the right hand side above is bounded by $ 2n^{-12} $. It follows from the union bound that
$\left\Vert \U\U^{T}\EE\right\Vert _{\infty}\le  C \sqrt{\frac{p_0 n \log n}{\Kmin^2}}$ w.h.p. Clearly, the same bound holds for
$\left\Vert \EE \U\U^{T}\right\Vert _{\infty}$. Finally, let $ K $ be the size of the cluster that node $ j $ is in. Observe that due to the structure of $ \U \U^\top $, we have
\begin{align*}
\left(\U\U^{T} \EE \U\U^{T}\right)_{i,j}
 & =  \sum_{l} \left( \U\U^\top \PP_\Omega(\S) \right)_{i,l} \left(\U\U^\top\right)_{l,j} \le \frac{1}{K} \cdot K \cdot \left\Vert\U\U^\top \PP_\Omega(\S) \right\Vert_\infty,
\end{align*}
which implies $ \left\Vert \U\U^{T} \EE \U\U^{T}\right\Vert_{\infty} \le \left\Vert\U\U^\top \PP_\Omega(\S) \right\Vert_\infty.  $ This completes the proof of the lemma.
\end{proof}

\section{Proof of Theorem \ref{thm:lowerbounds}}\label{sec:proof_lower}

We use a standard information theoretical argument, which improves upon a related proof by \citet{chaudhuri}. Let $ K $ be the size of the clusters (which are assumed to have equal size). For simplicity we assume $ n/K $ is an integer. Let $ \mathcal{F} $ be the set of all possible partition of $ n $ nodes into $ n/K $ clusters of equal size $ K $. Using Stirling's approximation, we have
$$
M\triangleq |\mathcal{F}| = \frac{1}{(n/K)!}{n \choose K}{n-K \choose K}\cdots{K \choose K} \ge \left(\frac{n}{3K}\right)^{n(1-\frac{1}{K})} \ge c_1^{\frac{1}{2} n},
$$
which holds for $ K=\Theta(n) $.

Suppose the clustering $ \Y $ is chosen uniformly at random from $ \mathcal{F} $, and the graph $ \A $ is generated from $ \Y $ according to the planted partition model with partial observations, where we use $ a_{ij} = ? $ for unobserved pairs. We use $ \mathbb{P}_{\A|\Y} $ to denote the distribution of $ \A $ given $ \Y $. Let $ \hat{\Y} $ be any measurable function of the observation $ \A $. A standard application of Fano's inequality and the convexity of the mutual information \citep{yang1999minimax} gives
\begin{align}
\label{eq:fano}
\sup_{Y \in \mathcal{F}}\mathbb{P}\left[\hat{\Y}\neq \Y \vert \Y\right] \ge 1 - \frac{M^{-2}\sum_{\Y^{(1)},\Y^{(2)}\in \mathcal{F}} D\left(\mathbb{P}_{\A|\Y^{(1)}}\Vert\mathbb{P}_{\A|\Y^{(2)}}\right)+\log 2}{\log M},
\end{align} 
where $ D(\cdot\Vert\cdot) $ denotes the KL-divergence. We now upper bound this divergence. Given $ \Y^{(l)} $, $ l=1,2 $, the $ a_{i,j} $'s are independent of each other, so we have
$$
D\left(\mathbb{P}_{\A|\Y^{(1)}}\Vert\mathbb{P}_{\A|\Y^{(2)}}\right) = \sum_{i,j}D\left(\mathbb{P}_{a_{i,j}|\Y^{(1)}}\Vert\mathbb{P}_{a_{i,j}|\Y^{(2)}}\right).
$$
For each pair $ (i,j) $, the KL-divergence is zero if $ y^{(1)}_{i,j} = y^{(2)}_{i,j}$, and otherwise satisfies
\begin{align*}
D\left(\mathbb{P}_{a_{i,j}|\Y^{(1)}}\Vert\mathbb{P}_{a_{i,j}|\Y^{(2)}}\right) 
& \le  p_0(1-\tau) \log \frac{p_0(1-\tau)}{p_0\tau} + p_0\tau\log \frac{p_0\tau}{p_0(1-\tau)} + (1-p_0) \log \frac{1-p_0}{1-p_0}\\
& =    p_0(1-2\tau)\log \frac{1-\tau}{\tau} \\
& \le  p_0(1-2\tau) \left(\frac{1-\tau}{\tau} - 1\right)\\
& \le  c_2 p_0(1-2\tau)^2,
\end{align*}
where $ c_2>0 $ is a universal constant and the last inequality holds under the assumption $ \tau>1/100 $. Let $ N $ be the number of pairs $ (i,j) $ such that $ y^{(1)}_{i,j} \neq y^{(2)}_{i,j}$. When $ K=\Theta(n) $, we have 
$$ N \le |\{(i,j):y^{(1)}_{i,j}=1 \} \cup \{(i,j):y^{(2)}_{i,j}=1 \} | \le n^2. $$ 
It follows that $ D\left(\mathbb{P}_{\A|\Y^{(1)}}\Vert\mathbb{P}_{\A|\Y^{(2)}}\right) \le N\cdot c_2 p_0(1-2\tau)^2 \le c_2 n^2 p_0(1-2\tau)^2. $
Combining pieces, for the left hand side of \eqref{eq:fano} to be less than $ 1/4 $, we must have 
$
p_0(1-2\tau)^2 \ge C \frac{1}{n}.
$

\vskip 0.2in
\bibliography{GraphClustering}

\end{document}